\documentclass[11pt]{article}
\usepackage[margin=1in]{geometry}
\usepackage{mathtools}
\usepackage{booktabs}
\usepackage{float}
\usepackage{hyperref}
\usepackage{mmll2}   
\hypersetup{colorlinks=true, allcolors=MidnightBlue}
\newcommand{\DSPACE}{\mathsf{DSPACE}}

\newcommand{\TC}{\mathsf{TC}}

\newcommand{\Pclass}{\mathsf{P}}
\newcommand{\CLL}{\mathrm{CLL}}
\newcommand{\SSL}{\mathrm{SSL}}
\newcommand{\COT}{\mathrm{CoT}}
\newcommand{\poly}{\mathrm{poly}}
\newcommand{\polylog}{\mathrm{polylog}}

\title{Chain-of-Thought and Compressed Looped Transformers:\\
A Memory-Budget Separation}
\author{Haozhou Zhang \\
Department of Mathematics and Statistics \\
University of Idaho \\
\texttt{zhan3063@vandals.uidaho.edu}}
\date{\today}

\begin{document}
\maketitle

\begin{abstract}
Chain-of-thought prompting and looped Transformers both give a fixed model more
test-time computation, but they differ in what they remember. Chain-of-thought
stores intermediate state in generated tokens that remain in the context,
whereas a looped Transformer carries state through recurrent hidden activations.
We argue that this persistent mutable memory is a central resource for test-time
reasoning.

We compare three memory regimes, the compressed latent loop, the full
sequence-state loop, and the chain-of-thought scratchpad. Our main result shows that a compressed loop is
limited by the size of its recurrent state. Running the loop longer adds
computation but does not by itself create a growing scratchpad, so a loop with a
small recurrent state remains a small-space reasoner even when run for many
steps. Under a standard complexity assumption, such loops cannot decide problems
that are $\Pclass$-complete under logspace reductions, whereas polynomial-length
chain-of-thought can.

The separation is specific to compressed loops, as full sequence-state loops
carry state at every input position and live in a memory-rich regime closer to
explicit scratchpads. Controlled pointer-chasing and associative-recall sweeps
illustrate this memory-budget view, with performance sensitive to whether the
persistent-state budget matches the task's working-memory demand.
\end{abstract}

\section{Introduction}

Chain-of-thought (CoT) prompting and recurrent-depth Transformers are two ways of
giving a fixed model more computation at test time. In CoT,
the model writes intermediate state into generated tokens and later reads those
tokens back from the context. In a looped Transformer, the model instead reuses
a block on hidden activations, carrying intermediate state through the recurrent
hidden state. These two mechanisms both add iteration, but they store
intermediate information in different places.

This distinction matters for serial reasoning. \citet{li2024chain} show that CoT
can supply sequential computation, since with enough generated steps
constant-depth Transformers simulate increasingly large Boolean circuits, and
\citet{merrill2024the} characterise the polynomial-length scratchpad regime as
$\Pclass$. Looped and latent-thought models have also been shown to be powerful,
but the strongest formal comparisons usually study \emph{memory-rich} loops: the
recurrent state contains a full hidden sequence, and therefore grows with the
input length or with appended scratch positions. In that regime, looped
Transformers can act as programmable computers \citep{giannou2023programmable},
simulate CoT by appending tokens \citep{saunshi2025reasoning}, and in some
parallel-evaluable settings outperform explicit token-by-token reasoning
\citep{cotvsloop2025}.

The question we study is what remains true when the recurrent state is
\emph{not} a full sequence. A loop may be run for many iterations, but if it
carries only a small persistent state from one iteration to the next, then
recurrence adds time without adding much mutable memory. We therefore separate
the amount of computation performed from the amount of state preserved across
steps.

Throughout, $n$ denotes the input length, $d$ the hidden width, $p$ the number
of bits used to store each activation, $s$ the number of compressed recurrent
slots, $T$ the number of loop iterations, and $\ell$ the number of generated
CoT tokens. We write $M$ for the number of persistent mutable bits
carried from one reasoning step to the next. With this accounting, a compressed
latent loop has $M = s\,d\,p$; a sequence-state loop has $M = n\,d\,p$, since it
stores a hidden vector at every input position; and a CoT
scratchpad has $M = \Theta(\ell \log|\Sigma|)$ ($\Theta(\ell)$ for a fixed
vocabulary), since generated tokens are appended to the context and re-read
later.

Our formal result isolates the compressed-loop regime. Write $\DSPACE(f)$ for
the class of languages a deterministic Turing machine decides using $O(f)$ work
space, $\polylog(n) = (\log n)^{O(1)}$, and
$\DSPACE(\polylog n) = \bigcup_{c \ge 1} \DSPACE\!\big((\log n)^{c}\big)$. A
compressed loop with $s$ recurrent slots, width $d$, precision $p$, and
$T$ iterations is then simulated in
\[
\DSPACE\!\big(O(s\,d\,p + \log n + \log T)\big).
\]
The simulation stores the recurrent state, streams over the read-only input
whenever attention needs to inspect it, and reuses the same scratch space at
each iteration. Thus, if the persistent state is polylogarithmic in the input length and the
loop runs for at most quasi-polynomially many steps (so the iteration counter
needs only polylogarithmic space), the loop remains a polylogarithmic-space
machine.

This gives a conditional separation from CoT in the memory-constrained regime.
Assuming the standard but unproven separation
$\Pclass \not\subseteq \DSPACE(\polylog n)$, such compressed loops
cannot decide $\Pclass$-complete serial tasks, whereas
polynomial-length CoT can under standard expressivity results. This separation
is specific to the compressed case. Full sequence-state loops carry $n\,d\,p$
persistent bits and lie in a different, memory-rich regime, so what the bound
isolates is the limited persistent memory rather than recurrence in general.

\medskip\noindent Concretely, the paper makes the following five contributions:
\begin{itemize}
\item We formalise persistent mutable memory as a resource for comparing CoT,
full sequence-state loops, and compressed latent loops.
\item We prove that compressed latent loops are small-space machines: a loop
with only a compressed recurrent state can be simulated by storing that state,
streaming over the read-only input, and reusing scratch space across iterations.
\item We show why this argument does not extend to full sequence-state loops in
the same way. Once the loop carries hidden state at every input position, the
corresponding memory budget is linear in the sequence length, placing these
models in the memory-rich regime studied by prior looped-Transformer work.
\item We derive a conditional separation between polylogarithmic-memory
compressed loops and polynomial-length CoT on serial polynomial-time tasks,
assuming polynomial time is not contained in polylogarithmic space.
\item We connect prior memory-token results on Sudoku to the persistent-memory
lens as external motivation, and provide two controlled finite-size probes, a
pointer-chasing slot sweep and an associative-recall state-size sweep.
\end{itemize}

\noindent The rest of the paper is structured as follows. \Cref{sec:related}
reviews related work, and \Cref{sec:models} defines the three reasoners and
their memory budgets. \Cref{sec:formal} develops the formal separation in three
steps: the small-space bound for compressed loops, the sequence-state endpoint,
and the separation from CoT, with detailed proofs deferred to
\Cref{app:lemma,app:upper,app:escape,app:crossover}. \Cref{sec:empirics}
presents the empirical studies, and \Cref{sec:discussion} discusses the scope of
the result.

\section{Related work}
\label{sec:related}

\paragraph{Expressivity of CoT.}
\citet{li2024chain} show CoT supplies sequential computation. With $\ell$
generated steps, constant-depth Transformers simulate Boolean circuits of size
$\ell$. \citet{merrill2024the} give a length-parameterised characterisation of
decoders with scratchpads, identifying the polynomial-step regime with $\Pclass$
under their model assumptions, and \citet{bavandpour2025lower} prove lower bounds
on the number of CoT steps required in the hard-attention setting. These works
identify the CoT scratchpad as a growing computational resource; in our notation
this is the endpoint $M_{\COT} = \Theta(\ell)$ for a fixed token alphabet.
\citet{yang2025pencil} instead modify the scratchpad interface, adding a
reduction rule that erases intermediate thoughts to achieve space-efficient
universal computation; our CoT endpoint is the standard append-only scratchpad,
so erasure or reduction occupies a different point on the same memory axis.

\paragraph{Looped and recurrent-depth Transformers.}
\citet{giannou2023programmable} construct looped Transformers as programmable
computers whose input acts as both instruction stream and memory.
\citet{saunshi2025reasoning} show looped Transformers solve
group-composition-style iterative tasks with logarithmically many loops and can
simulate CoT using appended scratch positions. \citet{luca2024simulation}
construct looped Transformers that simulate classical graph algorithms (BFS, DFS,
Dijkstra, Kosaraju) and, with extra attention heads, are Turing-complete, while
\citet{yang2024looped} show looped Transformers can emulate iterative learning
algorithms with substantially fewer parameters than standard Transformers on
data-fitting tasks. These constructions are memory-rich in the sense relevant
here: the loop maintains a full sequence state, padding or scratch positions, or
graph-interaction state, rather than only a fixed compressed latent vector, which
our \Cref{prop:escape} recovers as the memory-rich endpoint.
\citet{merrill2026a,merrill2026exact} further study depth, padding, and looping
as parallelisable forms of test-time compute, with log-depth already adding
expressive power for regular languages and graph connectivity, while padding
combined with looping yields exact characterisations in threshold-circuit classes.

\paragraph{Memory tokens and the depth-state trade-off.}
\citet{sapunov2026utm} studies a single-block Universal Transformer with adaptive
computation time and a variable number $r$ of learned memory tokens on
Sudoku-Extreme. The reported threshold between zero and eight memory tokens, the
plateau, the dilution boundary, and the substitution between memory size and halt
depth are suggestive of a persistent-memory effect, and our space-complexity lens
predicts such thresholds when the effective working-state demand exceeds the
persistent state available to the model. The architecture also carries
sequence-state activations, however, so it is not a clean instance of our
compressed-loop model, and we treat it as external motivation rather than a
diagnostic test of the separation. \citet{lu2025latent} analyse the
depth-recurrent model of \citet{geiping2026scaling}, finding only marginal gains
from added recurrence depth and limited evidence of interpretable latent CoT.
Relatedly, \citet{blayney2026mechanistic} mechanistically analyse
looped-reasoning models, reporting cyclic fixed points and attention
stabilisation across recurrences. These works support the view that recurrence
can reorganise computation without necessarily increasing the amount of
persistent state.

\paragraph{Recall-memory trade-offs in efficient attention.}
A separate line of work measures the state needed for associative recall in
efficient sequence models. \citet{arora2024zoology} introduce multi-query
associative recall as a diagnostic for efficient language models and attribute
much of the gap between attention and gated-convolution models to recall
ability, while \citet{arora2024simple} identify a recall-throughput trade-off in
which fixed-state recurrent alternatives recover recall as the recurrent state
or feature dimension grows, at a corresponding memory cost.
\citet{okpekpe2025revisiting} further show that the learning rate strongly
affects recurrent-model recall, so reported recall gaps partly reflect
optimisation. Our associative-recall sweep in \Cref{sec:gla} uses this setting
as a controlled probe of the same persistent-memory axis; the formal separation
concerns compressed looped Transformers and append-only CoT.

\paragraph{CoT versus latent or looped thought.}
\citet{cotvsloop2025} formally compare CoT with latent-thought reasoning,
including looped and continuous hidden-state computation, the latter exemplified
by the continuous-thought model of \citet{hao2025training}. Their results
separate the two interfaces along two axes, with latent thought holding an
advantage for parallel computation and CoT a separate advantage from stochastic
decoding and approximate counting. We study a different axis. Holding attention
on deterministic serial computation, we ask what happens when the looped reasoner
does not carry a full hidden sequence but only a bounded persistent state; in
that memory-constrained regime, the limiting resource is persistent memory rather
than parallel depth.

\paragraph{Circuit complexity of Transformers.}
\citet{merrill2023logspace} place log-precision Transformers in logspace-uniform
$\TC^0$, providing the circuit-complexity background for the comparisons above.
We use the same log-precision, logspace-uniform setting to upper-bound compressed
latent loops by their persistent mutable memory budget.

\section{Models and memory accounting}
\label{sec:models}

We now define the three idealised reasoners used in the separation. As above,
$n$ is the input length, $d$ the hidden width, and $p$ the number of bits used
to store each activation; let $\mathcal{F}_p$ denote the $p$-bit fixed-point
activation set, defined precisely in \Cref{app:model}. The memory budget $M$
counts mutable activation state that persists across reasoning steps; it does
not count the fixed parameters, the read-only input $x \in \Sigma^n$, or
temporary workspace recomputed within a step. A \emph{block} $B_\theta$ is a
constant number of multi-head attention and multi-layer-perceptron (MLP)
sublayers at precision $p$; we assume the log-precision regime $p = O(\log n)$
and a logspace-uniform parameter family, following \citet{merrill2023logspace}.

\begin{definition}[Compressed latent loop {$\CLL[s,d,p,T]$}]
\label{def:cll}
The mutable state is $z_t \in (\mathcal{F}_p^d)^s$, holding $s$ persistent slots,
initialised by $z_0 = \mathrm{init}_\theta(x)$, which is either fixed learned
memory or computable from the read-only input in $O(s\,d\,p + \log n)$ space. The
update
\[
z_{t+1} = B_\theta(z_t \,;\, x), \qquad t = 0, \dots, T-1,
\]
lets the $s$ state-slots act as attention queries against the $n$ read-only
input tokens as keys and values. The input side is read-only at every sublayer,
where the key and value for input position $i$ are logspace-computable functions
of $x_i$, the position index, and the parameters, and the block never materialises
or updates an $n$-position hidden sequence. Only the $s$ latent slots are mutable
across sublayers and iterations; slot-to-slot self-attention among them is
permitted, since those slots are already stored in $O(s\,d\,p)$ space. The output
is decoded from $z_T$, and no per-input-token persistent state is carried between
iterations. Its persistent memory is $M_{\CLL} = s\,d\,p$ bits.
\end{definition}

\begin{definition}[Sequence-state loop {$\SSL[d,p,T]$}]
\label{def:ssl}
The mutable state is $H_t \in (\mathcal{F}_p^d)^n$, the full length-$n$ hidden
sequence, initialised by $H_0 = \mathrm{embed}_\theta(x)$. Each update
$H_{t+1} = B_\theta(H_t)$ is a full self-attention pass that updates all $n$
positions, with the input embeddings and positions available as fixed features.
This is the standard looped Transformer of
\citet{giannou2023programmable, saunshi2025reasoning}; if padding or scratch
positions are appended, they are included in the sequence length $n$. Its
persistent memory is $M_{\SSL} = n\,d\,p$ bits.
\end{definition}

\begin{definition}[CoT scratchpad {$\COT[d,p,\ell]$}]
\label{def:cot}
The model autoregressively emits tokens $y_1, \dots, y_\ell \in \Sigma$, each
\[
y_t = \mathrm{decode}\big(B_\theta([x \,;\, y_{<t}])\big),
\]
appended to the context and re-read at later steps. Its persistent memory is the
scratchpad, $M_{\COT} = \Theta(\ell \log|\Sigma|)$ bits ($\Theta(\ell)$ for a
fixed vocabulary); we count the externally
written scratchpad tokens, not implementation-level caches used to accelerate
decoding.
\end{definition}

For the comparison below, the relevant difference between the three is the
persistent mutable memory carried across reasoning steps, fixed by the slot count
in a compressed loop, growing with the sequence length in a sequence-state loop,
and growing with the scratchpad in chain-of-thought.

\section{Space bounds and the memory-budget separation}
\label{sec:formal}

We prove the separation in three steps: compressed loops admit a small-space
simulation; full sequence-state loops instead carry linear state; and combining
these with standard chain-of-thought expressivity separates the two in the
compressed-memory regime.

\subsection{Compressed loops admit a small-space simulation}
\label{sec:upper}

The memory accounting of \Cref{sec:models} becomes useful once it is turned into
a space simulation. Since a compressed loop carries no mutable state at the input
positions, a simulator only needs to store the $s$ recurrent slots. Attention
over the input can be recomputed by streaming over the read-only input tape, and
the scratch space can be reused across slots, layers, and iterations. Thus
recurrence adds time, but not persistent mutable memory, apart from the counter
used to track the loop steps. The full one-block simulation is in \Cref{app:lemma}
and the iteration argument in \Cref{app:upper}.

\begin{lemma}[One compressed-loop update in small space]
\label{lem:onestep}
Under the log-precision, logspace-uniform regime, one update
$z \mapsto B_\theta(z \,;\, x)$ of a $\CLL[s,d,p,T]$ is computable in
deterministic work-space $O(s\,d\,p + \log n)$, given $z$ stored on the work
tape and $x$ on a read-only input tape.
\end{lemma}

\begin{proof}[Proof intuition]
A compressed loop never materialises an $n$-position hidden sequence. To update
one latent slot, the simulator streams over the read-only input, computing the
attention numerator and denominator with a constant number of passes and a few
$p$-bit registers. Once the slot's readout has been accumulated, the MLP and
residual updates are local computations on $O(dp)$ bits. The same scratch is
reused for every slot, so the only persistent workspace is the $s\,d\,p$-bit
latent state, plus $O(\log n)$ addressing and arithmetic scratch.
\Cref{app:lemma} gives the full finite-precision simulation.
\end{proof}

\begin{theorem}[Compressed loops are space-bounded]
\label{thm:upper}
For any $\CLL[s,d,p,T]$ in the log-precision logspace-uniform regime,
\[
\CLL[s,d,p,T] \subseteq \DSPACE\!\big(O(s\,d\,p + \log n + \log T)\big).
\]
In particular, if $s\,d\,p = \polylog(n)$ and $\log T = \polylog(n)$ (e.g.\ for
polynomially many iterations), then $\CLL \subseteq \DSPACE(\polylog n)$.
\end{theorem}

The displayed bound assumes $O(p)$-space $p$-bit arithmetic; under general
logspace-uniform arithmetic it carries an additive $\poly(p)$ term, still
polylogarithmic when $p = O(\log n)$ (see \Cref{app:model}).

\begin{proof}[Proof intuition]
The simulator runs the loop literally: it stores the current compressed state,
computes the next state using \Cref{lem:onestep}, overwrites the old state, and
advances a $\log T$-bit counter. The input may be rescanned at every iteration,
but because it is read-only this increases time, not workspace. Iteration thus
adds a counter but no persistent mutable memory. The full proof is in
\Cref{app:upper}.
\end{proof}

\subsection{The sequence-state endpoint}
\label{sec:escape}

\Cref{sec:upper} bounded compressed loops by the space needed to store their
recurrent slots. For sequence-state loops, the same accounting gives a different
scale. The recurrent state is the whole hidden sequence: one $d$-dimensional
$p$-bit vector for each of the $n$ positions, so the persistent state has
$n\,d\,p$ bits. A simulator can store this sequence and run the same
step-by-step argument as before, but the resulting bound is linear in the
sequence length. This is the memory-rich regime used by standard
looped-Transformer expressivity constructions. \Cref{app:escape} gives the proof
of \Cref{prop:escape}.

\begin{proposition}[Sequence-state loops are memory-rich]
\label{prop:escape}
For $\SSL[d,p,T]$, the persistent memory is $M_{\SSL} = n\,d\,p$, and the
simulation argument of \Cref{thm:upper} gives
\[
\SSL[d,p,T] \subseteq \DSPACE(O(n\,d\,p + \log n + \log T)).
\]
For fixed $d,p$ this is $\Theta(n)$ space, linear in the sequence length up to
the width and precision factors $d$ and $p$. This linear-state endpoint is
genuinely memory-rich. The programmable-computer construction of
\citet{giannou2023programmable}, combined with the CoT simulation of
\citet{saunshi2025reasoning} and the CoT-to-circuits result of
\citet{li2024chain}, places $\Pclass$-complete tasks such as the circuit value
problem (CVP) within the sequence-state endpoint. Here $n$ denotes the total
sequence length available to the loop, including any padding or dummy scratch
positions. Hence, under $\Pclass \not\subseteq \DSPACE(\polylog n)$,
$\SSL \not\subseteq \DSPACE(\polylog n)$.
\end{proposition}

What separates the two regimes is the scale of the persistent state. Compressed
loops keep $s\,d\,p$ recurrent bits;
sequence-state loops keep $n\,d\,p$ recurrent bits. Increasing the number of
iterations affects time and adds only the loop counter, while the amount of
persistent mutable state is fixed by the recurrent representation.
\Cref{sec:crossover} uses this difference to state the separation from CoT in
the compressed-memory regime.

\subsection{Separation from chain-of-thought}
\label{sec:crossover}

The small-space simulation now gives the separation. A compressed loop whose
persistent state is polylogarithmic, run for at most quasi-polynomially many
steps, is a polylogarithmic-space machine. A chain-of-thought scratchpad grows in a
different way. Each generated token is written into the context and remains
available to later steps. Combining this space bound with standard expressivity
results for polynomial-length CoT \citep{li2024chain, merrill2024the} yields a
conditional separation on $\Pclass$-complete serial tasks. The formal proof is
in \Cref{app:crossover}.

\begin{theorem}[Compressed-loop/chain-of-thought separation]
\label{thm:crossover}
Assume $\Pclass \not\subseteq \DSPACE(\polylog n)$, and let $\mathcal{L}$ be
$\Pclass$-complete under logspace reductions. No logspace-uniform compressed-loop family with
\[
s\,d\,p = \polylog(n), \qquad \log T = \polylog(n)
\]
decides $\mathcal{L}$. By contrast, under standard chain-of-thought expressivity
results, $\mathcal{L}$ is decidable by a chain-of-thought family with
polynomially many generated tokens.
\end{theorem}

\emph{Proof intuition.} By \Cref{thm:upper}, a compressed loop with
$s\,d\,p = \polylog(n)$ and $\log T = \polylog(n)$ lies in $\DSPACE(\polylog n)$.
If such a loop decided a $\Pclass$-complete language, then every language in
$\Pclass$ would be decidable in polylogarithmic space, contradicting the
assumption. CoT is not subject to this compressed-state bound: its scratchpad
grows as tokens are generated, and polynomial-length scratchpads reach the
polynomial-time regime under the standard expressivity results
\citep{li2024chain, merrill2024the}. The full proof is in \Cref{app:crossover}.

The theorem addresses the compressed-state endpoint. The memory-rich endpoint
has a different budget, since sequence-state loops carry $n\,d\,p$ persistent bits.
This is the setting of the formal latent-thought comparison of
\citet{cotvsloop2025}, where looped or latent thought can outperform CoT on
parallel-evaluable computations. The two conclusions therefore apply to
different persistent-state budgets.

\section{Empirical studies}
\label{sec:empirics}

The space bound suggests a concrete diagnostic. In a compressed loop, extra
iterations let the model revisit the input and refine the same recurrent state,
but they add no new places to store independent pieces of information. If this
accounting carries over to trained finite models, then a task whose
working-memory load can be dialled up should show a state-budget transition, with
more persistent state helping as the load grows and added recurrence failing to
relieve an undersized-state bottleneck.

We study two controlled settings in which the task load and the persistent-state
budget vary independently. In pointer chasing the load is the number of
independent chains tracked to a single final decode, and the budget is the number
of compressed slots; in associative recall the load is the number of key-value
pairs, and the budget is the recurrent state dimension. Both are finite-size
probes of that prediction, with recurrence depth and training budget held fixed
so that only the state knob varies. All seed-level numbers behind the
two sweeps, with no run excluded from any reported aggregate, are tabulated in
\Cref{app:seeds}.

\subsection{A controlled slot sweep}
\label{sec:poc}

Pointer chasing varies concurrent working-memory demand directly, with the number
of compressed slots $s$ as the persistent-memory knob we sweep. The
task is parameterised by $k$, the number of independent
pointer chains whose states must be maintained at once. Each instance gives a
functional graph on $n=16$ cells, $k$ start cells, and asks for the cell each
chain reaches after a fixed number of hops. Under the single-final-decode
protocol used here, a direct concurrent algorithm stores all $k$ current chain
states at once, $\Theta(k\log n)$ bits; the experiment tests whether the
compressed slots behave like such concurrent working memory. We compare
compressed latent loops
$\CLL[s]$ with $s$ persistent slots against a sequence-state loop ($\SSL$,
$M=\Theta(n)$) control, sweeping $s \in \{1,2,4,8,16\}$ and
$k \in \{1,2,4,8\}$, $T=16$ iterations, ten seeds. Both models instantiate the
reasoners of \Cref{sec:models}. The compressed loop carries only its $s$ slots
across the iterations, attending at each step to a cell table that is recomputed
from the read-only input and never updated, so its persistent budget is
$M_{\CLL}=s\,d\,p$ in the sense of \Cref{def:cll}, and each chain is read out by a
query against the slots after the final iteration. The sequence-state control
instead updates a hidden vector at all $n$ cells, matching \Cref{def:ssl}.

The finite-size prediction is deliberately one-sided. When the persistent state falls below
the task's concurrent-state demand, recurrence alone cannot repair the
bottleneck, while sufficient nominal state does not by itself guarantee that
training finds a solution.
The slot count $s$ is a register-count bottleneck for the $k$ chains, so the
subdiagonal regime $s < k$ is the stress test. In this experiment $s$ should be
read as a register count, not a literal bit budget. Formally, $s < k$ is a bit-space lower bound only when
$d\,p = \Theta(\log n)$ and each slot holds $O(1)$ chain states, whereas here
$n = 16$ makes $\log n$ small relative to $d\,p$. The theory does not assert a
trainable threshold at $s = k$; its converse (\Cref{lem:budget}) is conditional
on sufficient iteration, routing, and trainability.

\Cref{fig:heatmap} shows the 10-seed sweep. In the subdiagonal regime $s<k$ the
collapse is robust, with mean cell accuracy $0.041$ and maximum cell-mean
$0.129$. The collapse is also low-variance
in the genuinely subcritical cells, with standard deviations between $0.00$ and
$0.02$ for $k\in\{2,4,8\}$ below the diagonal, so the failure below the slot
budget is not a seed accident. The sequence-state control is essentially
saturated, reaching $0.999$ mean accuracy for $k=2,4,8$; at $k=1$, nine of ten
$\SSL$ seeds reach $\approx 1.00$ while one fails to train, giving a mean of
$0.936$. We therefore treat the $\SSL$ result as a saturated control up to a
single optimisation outlier.

The 10-seed sweep also shows that the above-threshold regime is not smooth. Some
cells are reliably solved ($k{=}2$ with $s\ge 2$, $k{=}4$ with $s\in\{4,8\}$),
but others are seed-dependent: for $k{=}1,s{=}1$ the runs split into six high
seeds and four low seeds with little mass between them, and the over-provisioned
$s{=}16$ column is unstable, with $k{=}8,s{=}16$ ranging from near failure to
near-perfect accuracy. These high-variance cells look less like noisy estimates
of a single mean and more like mixtures of successful and failed training runs.
Our interpretation is that, once the hard memory bottleneck is removed, the
remaining failures are mostly about whether training finds a usable routing
scheme, not a contradiction of the lower-bound direction. It is the same
phenomenon as the router-initialisation trap of \citet{sapunov2026utm}, where
above the memory threshold success still depends on finding the right recurrent
routing solution.

\begin{figure}[H]
\centering
\includegraphics[width=0.6\textwidth]{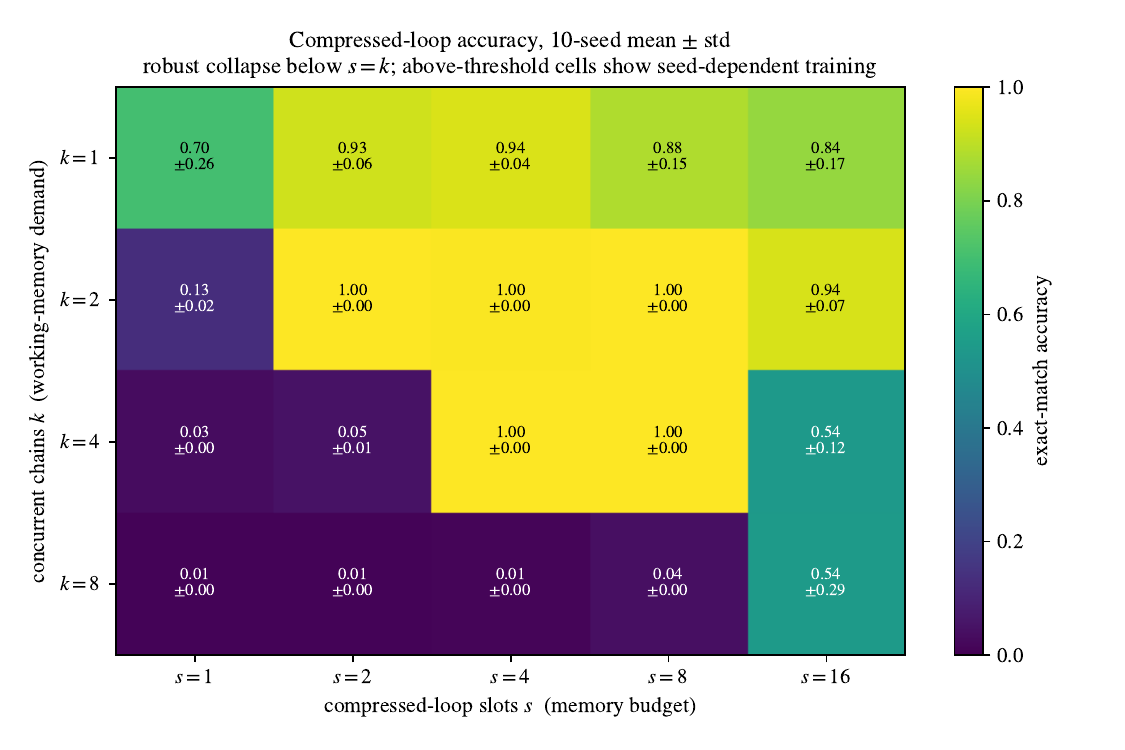}
\caption{Compressed-loop accuracy on $k$-concurrent pointer chasing, as a
function of persistent slots $s$ and concurrency $k$. Each cell shows exact-match
accuracy, mean $\pm$ standard deviation over $10$ seeds. The subdiagonal cells
$s<k$ collapse to low accuracy (mean $0.041$, maximum cell-mean $0.129$), while the
sequence-state loop control is essentially saturated: it reaches $\approx 1.00$
for $k=2,4,8$ and $0.94$ for $k=1$, where one of ten seeds failed to train. Above
and near the diagonal, several cells show seed-dependent success/failure
behaviour; these optimisation effects are not claimed as a tight achievability
threshold. The full ten-seed table for every compressed-loop cell and the
sequence-state control appears in
\Cref{tab:pointer-cll-seeds,tab:pointer-ssl-seeds}.}
\label{fig:heatmap}
\end{figure}

\subsection{Associative recall with finite recurrent state}
\label{sec:gla}

\begin{figure}[H]
\centering
\includegraphics[width=0.6\textwidth]{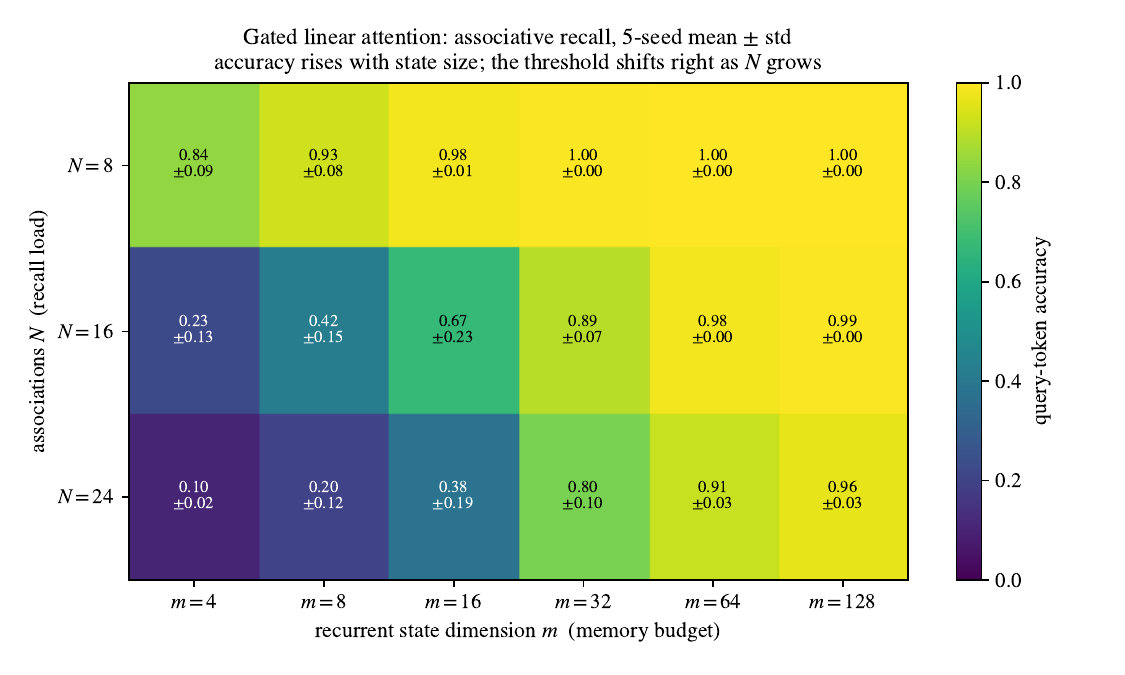}
\caption{Gated linear-attention query-token accuracy on associative recall with
$N$ key-value pairs, as a function of recurrent state dimension $m$, mean $\pm$
standard deviation over five seeds. The mean accuracy rises monotonically with
$m$ at each $N$, and the threshold shifts right as $N$ grows. Full softmax
attention is essentially saturated (mean $\ge 0.999$ in every cell); vanilla
linear attention remains low for $N{=}16,24$, even when taking the best $m$ in
the sweep. \Cref{app:seeds} reports the full five-seed table, the per-seed
$m_{\mathrm{crit}}$ values, and the control summaries.}
\label{fig:gla}
\end{figure}

As a second controlled check, and following the recall-memory literature on
efficient attention \citep{arora2024zoology, arora2024simple}, we test the same
memory-budget prediction in an associative-recall setting. Each example contains
$N$ key-value associations, and the model must retrieve the correct value from a
query. We vary the state dimension $m$, so the recurrent state size is
proportional to $m\,d$ for fixed width $d$, in a gated linear-attention model;
full softmax attention serves as a memory-rich control, and vanilla linear
attention as a compressed-state negative control. For each seed and association
load $N$, let $m_{\mathrm{crit}}$ be the smallest state dimension $m$ whose
accuracy reaches $0.9$.

The result is a state-dimension crossover (\Cref{fig:gla}). The five-seed mean
accuracy increases monotonically with $m$ for each association load $N$, and the
transition shifts to larger $m$ as $N$ increases. Full softmax attention is
essentially saturated in every setting. Vanilla linear attention, even taking
the best $m$ in the sweep, solves only the smallest association load and remains
low for larger loads. As in the pointer-chasing sweep, the saturated and failed
cells are stable across seeds, whereas transition cells show seed-dependent
success/failure behaviour. We therefore report $m_{\mathrm{crit}}$ as a range
across seeds rather than a sharp threshold: $N{=}8$ has median $8$ and range $4$
to $16$, $N{=}16$ has median $32$ and range $16$ to $64$, and $N{=}24$ has
median $64$ and range $32$ to $128$.

\section{Discussion and limitations}
\label{sec:discussion}

The separation suggests a simple design principle. Test-time recurrence should
be specified together with the state that persists across recurrences. A loop
whose state is compressed can spend more time refining the same memory, but it
does not acquire the growing scratchpad that CoT obtains by writing tokens. Full
sequence-state loops and padded latent-thought constructions sit in a different
regime, because their recurrent state scales with the sequence length or with
added scratch positions. Persistent memory is therefore a resource of test-time
compute alongside depth, width, and generated-token length.

For practical looped language models \citep{ouro2025, jeddi2026loopformer}, this
separates two notions of memory efficiency. Sharing a key-value (KV) cache across
loop depth, as in \citet{melt2026}, decouples memory from the number of recurrent
iterations, but
it does not by itself make the persistent state small relative to the input or
context length. The compressed-loop bound applies most directly to models whose
recurrent state is capped, pooled, low-rank, or otherwise sublinear in context
length. A concrete experimental target is to expose the persistent-state budget
as a controllable knob (the number of memory tokens, recurrent slots, state
dimension, or cache rank) and measure how the critical budget scales with task
load. The pointer-chasing and associative-recall sweeps in \Cref{sec:empirics}
are small examples of this type of measurement. Concurrent empirical work on
adaptive looping and gated memory banks \citep{frey2026adaptive} asks a related
architectural question, whether a looped Transformer should refine its hidden
state for longer or add explicit learned storage; our separation isolates a
complementary resource, the mutable activation state that persists across
recurrences.

The experiments also point to an optimisation question. Once the state budget is
large enough, success is not automatic, as threshold cells can be seed-dependent
and over-provisioned memory can make routing unstable. Future work should separate
capacity from trainability by measuring not only whether a state budget is
sufficient in principle, but how reliably a model learns to use that state.
Mechanistic analyses of routing, halting, memory-token specialisation, and
recurrent fixed points \citep{blayney2026mechanistic} would make this account
more predictive.

\paragraph{Limitations and next steps.}
The formal separation is asymptotic and conditional. It relies on the standard
assumption $\Pclass \not\subseteq \DSPACE(\polylog n)$, the log-precision,
logspace-uniform model, and a compressed-loop definition in which the persistent
recurrent state is capped independently of the input length. It does not
lower-bound full sequence-state loops, loops with appended scratch positions,
stochastic decoding, or architectures whose persistent state grows with the
context. The experiments are finite-size probes of the state-budget effect
rather than tests of the asymptotic theorem, and they do not prove task-level
lower bounds for Sudoku, pointer chasing, or associative recall. A natural next
step is to prove lower bounds for explicit task families and to scale the
controlled experiments until the bit-level memory regime itself becomes active;
the intermediate budgets between the $s\,d\,p$ compressed and $n\,d\,p$
sequence-state regimes, where many practical looped language models sit, are the
most relevant targets.

The point, then, is to separate the resources these interfaces provide, not to
rank looped Transformers against chain-of-thought. Recurrence,
generated tokens, and full sequence state are different ways of buying
test-time computation, but they buy different resources. Depth determines how
many updates a model can perform; persistent memory determines what those
updates can continue to remember. Any comparison of test-time reasoning
interfaces should therefore specify both the computation performed and the
state preserved across steps.

\clearpage
\bibliographystyle{ims}
\bibliography{refs}

\begin{thebibliography}{24}
\expandafter\ifx\csname natexlab\endcsname\relax\def\natexlab#1{#1}\fi
\expandafter\ifx\csname url\endcsname\relax
  \def\url#1{\texttt{#1}}\fi
\expandafter\ifx\csname urlprefix\endcsname\relax\def\urlprefix{URL }\fi

\bibitem[{Amiri et~al.(2025)Amiri, Huang, Rofin and Hahn}]{bavandpour2025lower}
\textsc{Amiri, A.}, \textsc{Huang, X.}, \textsc{Rofin, M.} and \textsc{Hahn,
  M.} (2025).
\newblock Lower bounds for chain-of-thought reasoning in hard-attention
  transformers.
\newblock In \textit{Forty-second International Conference on Machine
  Learning}.
\newline\urlprefix\url{https://openreview.net/forum?id=Oh9sG5ae2b}

\bibitem[{Arora et~al.(2024{\natexlab{a}})Arora, Eyuboglu, Timalsina, Johnson,
  Poli, Zou, Rudra and Re}]{arora2024zoology}
\textsc{Arora, S.}, \textsc{Eyuboglu, S.}, \textsc{Timalsina, A.},
  \textsc{Johnson, I.}, \textsc{Poli, M.}, \textsc{Zou, J.}, \textsc{Rudra, A.}
  and \textsc{Re, C.} (2024{\natexlab{a}}).
\newblock Zoology: Measuring and improving recall in efficient language models.
\newblock In \textit{The Twelfth International Conference on Learning
  Representations (ICLR)}.
\newline\urlprefix\url{https://openreview.net/forum?id=LY3ukUANko}

\bibitem[{Arora et~al.(2024{\natexlab{b}})Arora, Eyuboglu, Zhang, Timalsina,
  Alberti, Zou, Rudra and Re}]{arora2024simple}
\textsc{Arora, S.}, \textsc{Eyuboglu, S.}, \textsc{Zhang, M.},
  \textsc{Timalsina, A.}, \textsc{Alberti, S.}, \textsc{Zou, J.},
  \textsc{Rudra, A.} and \textsc{Re, C.} (2024{\natexlab{b}}).
\newblock Simple linear attention language models balance the recall-throughput
  tradeoff.
\newblock In \textit{Forty-first International Conference on Machine Learning
  (ICML)}.
\newline\urlprefix\url{https://openreview.net/forum?id=e93ffDcpH3}

\bibitem[{{Back de Luca} and Fountoulakis(2024)}]{luca2024simulation}
\textsc{{Back de Luca}, A.} and \textsc{Fountoulakis, K.} (2024).
\newblock Simulation of graph algorithms with looped transformers.
\newblock In \textit{Forty-first International Conference on Machine Learning
  (ICML)}.
\newline\urlprefix\url{https://openreview.net/forum?id=aA2326y3hf}

\bibitem[{Blayney et~al.(2026)Blayney, Arroyo, Obando-Ceron, Castro, Courville,
  Bronstein and Dong}]{blayney2026mechanistic}
\textsc{Blayney, H.}, \textsc{Arroyo, {\'A}.}, \textsc{Obando-Ceron, J.},
  \textsc{Castro, P.~S.}, \textsc{Courville, A.}, \textsc{Bronstein, M.~M.} and
  \textsc{Dong, X.} (2026).
\newblock A mechanistic analysis of looped reasoning language models.
\newblock \textit{arXiv preprint arXiv:2604.11791} .

\bibitem[{{Conchello Vendrell} et~al.(2026){Conchello Vendrell}, Masdemont,
  Grillo, Ros-Giralt, Behboodi and Massoli}]{melt2026}
\textsc{{Conchello Vendrell}, V.}, \textsc{Masdemont, A.~P.}, \textsc{Grillo,
  N.}, \textsc{Ros-Giralt, J.}, \textsc{Behboodi, A.} and \textsc{Massoli,
  F.~V.} (2026).
\newblock Memory-efficient looped transformer: Decoupling compute from memory
  in looped language models.
\newblock \textit{arXiv preprint arXiv:2605.07721} .

\bibitem[{Frey et~al.(2026)Frey, Shomali, Bashir, Berghaus, Koehler and
  Ali}]{frey2026adaptive}
\textsc{Frey, M.}, \textsc{Shomali, B.}, \textsc{Bashir, A.~H.},
  \textsc{Berghaus, D.}, \textsc{Koehler, J.} and \textsc{Ali, M.} (2026).
\newblock Adaptive loops and memory in transformers: Think harder or know more?
\newblock In \textit{Workshop on Latent {\&} Implicit Thinking {\textendash}
  Going Beyond {CoT} Reasoning}.
\newline\urlprefix\url{https://openreview.net/forum?id=F87X9c107e}

\bibitem[{Geiping et~al.(2025)Geiping, McLeish, Jain, Kirchenbauer, Singh,
  Bartoldson, Kailkhura, Bhatele and Goldstein}]{geiping2026scaling}
\textsc{Geiping, J.}, \textsc{McLeish, S.~M.}, \textsc{Jain, N.},
  \textsc{Kirchenbauer, J.}, \textsc{Singh, S.}, \textsc{Bartoldson, B.~R.},
  \textsc{Kailkhura, B.}, \textsc{Bhatele, A.} and \textsc{Goldstein, T.}
  (2025).
\newblock Scaling up test-time compute with latent reasoning: A recurrent depth
  approach.
\newblock In \textit{The Thirty-ninth Annual Conference on Neural Information
  Processing Systems}.
\newline\urlprefix\url{https://openreview.net/forum?id=S3GhJooWIC}

\bibitem[{Giannou et~al.(2023)Giannou, Rajput, Sohn, Lee, Lee and
  Papailiopoulos}]{giannou2023programmable}
\textsc{Giannou, A.}, \textsc{Rajput, S.}, \textsc{Sohn, J.-y.}, \textsc{Lee,
  K.}, \textsc{Lee, J.~D.} and \textsc{Papailiopoulos, D.} (2023).
\newblock Looped transformers as programmable computers.
\newblock In \textit{International Conference on Machine Learning (ICML)}.
\newline\urlprefix\url{https://proceedings.mlr.press/v202/giannou23a.html}

\bibitem[{Hao et~al.(2025)Hao, Sukhbaatar, Su, Li, Hu, Weston and
  Tian}]{hao2025training}
\textsc{Hao, S.}, \textsc{Sukhbaatar, S.}, \textsc{Su, D.}, \textsc{Li, X.},
  \textsc{Hu, Z.}, \textsc{Weston, J.~E.} and \textsc{Tian, Y.} (2025).
\newblock Training large language models to reason in a continuous latent
  space.
\newblock In \textit{Second Conference on Language Modeling (COLM)}.
\newline\urlprefix\url{https://openreview.net/forum?id=Itxz7S4Ip3}

\bibitem[{Jeddi et~al.(2026)Jeddi, Ciccone and Taati}]{jeddi2026loopformer}
\textsc{Jeddi, A.}, \textsc{Ciccone, M.} and \textsc{Taati, B.} (2026).
\newblock Loopformer: Elastic-depth looped transformers for latent reasoning
  via shortcut modulation.
\newblock In \textit{The Fourteenth International Conference on Learning
  Representations (ICLR)}.
\newline\urlprefix\url{https://openreview.net/forum?id=RzYXb5YWBs}

\bibitem[{Li et~al.(2024)Li, Liu, Zhou and Ma}]{li2024chain}
\textsc{Li, Z.}, \textsc{Liu, H.}, \textsc{Zhou, D.} and \textsc{Ma, T.}
  (2024).
\newblock Chain of thought empowers transformers to solve inherently serial
  problems.
\newblock In \textit{The Twelfth International Conference on Learning
  Representations (ICLR)}.
\newline\urlprefix\url{https://openreview.net/forum?id=3EWTEy9MTM}

\bibitem[{Lu et~al.(2025)Lu, Yang, Lee, Li and Liu}]{lu2025latent}
\textsc{Lu, W.}, \textsc{Yang, Y.}, \textsc{Lee, K.}, \textsc{Li, Y.} and
  \textsc{Liu, E.} (2025).
\newblock Latent chain-of-thought? decoding the depth-recurrent transformer.
\newblock \textit{arXiv preprint arXiv:2507.02199} .

\bibitem[{Merrill and Sabharwal(2023)}]{merrill2023logspace}
\textsc{Merrill, W.} and \textsc{Sabharwal, A.} (2023).
\newblock The parallelism tradeoff: Limitations of log-precision transformers.
\newblock \textit{Transactions of the Association for Computational Linguistics
  (TACL)} \textbf{11} 531--545.
\newblock ArXiv:2207.00729.

\bibitem[{Merrill and Sabharwal(2024)}]{merrill2024the}
\textsc{Merrill, W.} and \textsc{Sabharwal, A.} (2024).
\newblock The expressive power of transformers with chain of thought.
\newblock In \textit{The Twelfth International Conference on Learning
  Representations (ICLR)}.
\newline\urlprefix\url{https://openreview.net/forum?id=NjNGlPh8Wh}

\bibitem[{Merrill and Sabharwal(2025{\natexlab{a}})}]{merrill2026exact}
\textsc{Merrill, W.} and \textsc{Sabharwal, A.} (2025{\natexlab{a}}).
\newblock Exact expressive power of transformers with padding.
\newblock In \textit{The Thirty-ninth Annual Conference on Neural Information
  Processing Systems (NeurIPS)}.
\newline\urlprefix\url{https://openreview.net/forum?id=O1abxStFcy}

\bibitem[{Merrill and Sabharwal(2025{\natexlab{b}})}]{merrill2026a}
\textsc{Merrill, W.} and \textsc{Sabharwal, A.} (2025{\natexlab{b}}).
\newblock A little depth goes a long way: The expressive power of log-depth
  transformers.
\newblock In \textit{The Thirty-ninth Annual Conference on Neural Information
  Processing Systems (NeurIPS)}.
\newline\urlprefix\url{https://openreview.net/forum?id=5pHfYe10iX}

\bibitem[{Okpekpe and Orvieto(2025)}]{okpekpe2025revisiting}
\textsc{Okpekpe, D.} and \textsc{Orvieto, A.} (2025).
\newblock Revisiting associative recall in modern recurrent models.
\newblock \textit{arXiv preprint arXiv:2508.19029} .

\bibitem[{Sapunov(2026)}]{sapunov2026utm}
\textsc{Sapunov, G.} (2026).
\newblock Universal transformers need memory: Depth-state trade-offs in
  adaptive recursive reasoning.
\newblock \textit{arXiv preprint arXiv:2604.21999} .

\bibitem[{Saunshi et~al.(2025)Saunshi, Dikkala, Li, Kumar and
  Reddi}]{saunshi2025reasoning}
\textsc{Saunshi, N.}, \textsc{Dikkala, N.}, \textsc{Li, Z.}, \textsc{Kumar, S.}
  and \textsc{Reddi, S.~J.} (2025).
\newblock Reasoning with latent thoughts: On the power of looped transformers.
\newblock In \textit{The Thirteenth International Conference on Learning
  Representations (ICLR)}.
\newline\urlprefix\url{https://openreview.net/forum?id=din0lGfZFd}

\bibitem[{Xu and Sato(2026)}]{cotvsloop2025}
\textsc{Xu, K.} and \textsc{Sato, I.} (2026).
\newblock A formal comparison between chain of thought and latent thought.
\newblock \textit{arXiv preprint arXiv:2509.25239} .
\newline\urlprefix\url{https://arxiv.org/abs/2509.25239}

\bibitem[{Yang et~al.(2025)Yang, Srebro, McAllester and Li}]{yang2025pencil}
\textsc{Yang, C.}, \textsc{Srebro, N.}, \textsc{McAllester, D.} and \textsc{Li,
  Z.} (2025).
\newblock {PENCIL}: Long thoughts with short memory.
\newblock In \textit{Forty-second International Conference on Machine Learning
  (ICML)}.
\newline\urlprefix\url{https://openreview.net/forum?id=6wglsDXIei}

\bibitem[{Yang et~al.(2024)Yang, Lee, Nowak and
  Papailiopoulos}]{yang2024looped}
\textsc{Yang, L.}, \textsc{Lee, K.}, \textsc{Nowak, R.~D.} and
  \textsc{Papailiopoulos, D.} (2024).
\newblock Looped transformers are better at learning learning algorithms.
\newblock In \textit{The Twelfth International Conference on Learning
  Representations (ICLR)}.
\newline\urlprefix\url{https://openreview.net/forum?id=HHbRxoDTxE}

\bibitem[{Zhu et~al.(2025)Zhu, Wang, Hua, Zhang, Li, Que, Wei, Wen, Yin, Xing,
  Li, Shi, Ma, Li, Kergan, Smith, Qu, Hui, Wu, Min, Huang, Zhou, Ye, Liu, Yang,
  Shi, Lin, Zhao, Cai, Zhang, Huang, Bengio and Eshraghian}]{ouro2025}
\textsc{Zhu, R.-J.}, \textsc{Wang, Z.}, \textsc{Hua, K.}, \textsc{Zhang, T.},
  \textsc{Li, Z.}, \textsc{Que, H.}, \textsc{Wei, B.}, \textsc{Wen, Z.},
  \textsc{Yin, F.}, \textsc{Xing, H.}, \textsc{Li, L.}, \textsc{Shi, J.},
  \textsc{Ma, K.}, \textsc{Li, S.}, \textsc{Kergan, T.}, \textsc{Smith, A.},
  \textsc{Qu, X.}, \textsc{Hui, M.}, \textsc{Wu, B.}, \textsc{Min, Q.},
  \textsc{Huang, H.}, \textsc{Zhou, X.}, \textsc{Ye, W.}, \textsc{Liu, J.},
  \textsc{Yang, J.}, \textsc{Shi, Y.}, \textsc{Lin, C.}, \textsc{Zhao, E.},
  \textsc{Cai, T.}, \textsc{Zhang, G.}, \textsc{Huang, W.}, \textsc{Bengio, Y.}
  and \textsc{Eshraghian, J.} (2025).
\newblock Scaling latent reasoning via looped language models.
\newblock \textit{arXiv preprint arXiv:2510.25741} .

\end{thebibliography}

\clearpage
\appendix
\section{Computational conventions}
\label{app:model}

All results use the following standard conventions.

\begin{itemize}
\item The input $x \in \Sigma^n$ is on a read-only tape addressed by an
$O(\log n)$-bit register; only the read/write work tape counts toward space.

\item Activations are $p$-bit fixed-point numbers in
$\mathcal{F}_p = \{a\,2^{-p} : a \in \mathbb{Z},\ |a| \le 2^{2p}\}$, so each uses
$\Theta(p)$ bits, and we work in the log-precision regime $p = O(\log n)$
\citep{merrill2023logspace}.

\item Arithmetic, softmax exponentials, and reciprocals are deterministic $p$-bit
operations computed by specified logspace-uniform circuits, so the simulator
reproduces the finite-precision block exactly with no approximation gap. In full
generality the space bounds carry an additive $\poly(p)$ term, still
polylogarithmic under $p = O(\log n)$; the sharper $O(s\,d\,p + \log n)$ one-step
bound of \Cref{lem:onestep} assumes $O(p)$-space arithmetic.

\item A block $B_\theta$ is a constant-size composition of multi-head attention
and MLP sublayers with logspace-uniform weights (each weight bit computable from
its index in $O(\log n)$ space).
\end{itemize}

The three reasoners are those of \Cref{def:cll,def:ssl,def:cot}. For language
recognition, a family decides $L \subseteq \Sigma^*$ if a fixed threshold on the
decoder output equals $\mathbbm{1}_{\{x \in L\}}$ for every input $x$.

\section{Proof of \texorpdfstring{\Cref{lem:onestep}}{Lemma 4.1}}
\label{app:lemma}

\noindent\textbf{\Cref{lem:onestep} (One compressed-loop update in small space).}\;\emph{Under the log-precision,
logspace-uniform regime, one update $z \mapsto B_\theta(z \,;\, x)$ of a
$\CLL[s,d,p,T]$ is computable in deterministic work-space $O(s\,d\,p + \log n)$,
given $z$ stored on the work tape and $x$ on a read-only input tape.}

\begin{proof}
The update $z \mapsto B_\theta(z\,;\,x)$ is a constant-depth composition of
attention and MLP sublayers; it suffices to bound one attention sublayer, since
the constantly many sublayers reuse the same scratch and add only a constant
factor. Fix a query slot $j \in [s]$ and a head $h \in [H]$ with head dimension
$d_h = d/H$.

\emph{Query.} The query $q_{j,h} = W_{Q,h}\,\mathrm{norm}(z_j) \in \mathcal{F}_p^{d_h}$ is
a fixed linear map of the slot vector $z_j$ ($d p$ bits, on the work tape),
computable in $O(d p)$ space.

\emph{Streaming the softmax readout.} Let $\mathcal{K} = [n] \sqcup [s]$ index the
$n$ read-only input keys and the $s$ old-slot keys. For $u \in \mathcal{K}$ let
$k_u, v_u$ be the corresponding finite-precision key and value: for an input
position they are logspace-computable from $x_i$, the position, and the
parameters; for a slot position they are read from the stored old-slot buffer
(\Cref{def:cll}). The readout is
\[
\mathrm{out}_{j,h} = \sum_{u \in \mathcal{K}} \alpha_u\, v_u,
\qquad
\alpha_u = \frac{\exp(\ell_u - m)}{\sum_{u' \in \mathcal{K}}\exp(\ell_{u'}-m)},
\quad
\ell_u = \frac{\langle q_{j,h}, k_u\rangle}{\sqrt{d_h}},
\]
with $m = \max_{u \in \mathcal{K}} \ell_u$. Finite-precision addition is not
associative, so we fix a deterministic summation order, by increasing index over
$\mathcal{K}$, using the same rounded fixed-point addition as the block; the block's
attention sums are defined in this order, so the simulator reproduces them
exactly. We compute the readout in three read-only passes over $\mathcal{K}$,
storing no per-key quantity:
\begin{itemize}
\item \emph{Pass 1} computes $m = \max_{u \in \mathcal{K}} \ell_u$: for each $u$ read
its source (an input token $x_i$ via the $O(\log n)$-bit address register, or a
stored slot), form $k_u$ in $O(dp)$ space, compute $\ell_u$ from $q_{j,h}$ and
$k_u$, and update the running maximum.
\item \emph{Pass 2} computes $Z = \sum_{u \in \mathcal{K}} \exp(\ell_u - m)$ in the
fixed order, recomputing each $\ell_u$, in one register holding a sum of up to
$n+s$ terms in $[0,1]$; this fits in $O(p + \log(n+s))$ bits, absorbed by
$O(sdp + \log n)$.
\item \emph{Pass 3} accumulates $\mathrm{out}_{j,h}$ into a $d_h$-dimensional
$\mathcal{F}_p$ register ($O(d_h p)$ bits): for each $u$ recompute $\ell_u$ (not
stored), form $\alpha_u = \exp(\ell_u - m)/Z$, and add $\alpha_u v_u$ in the same
order.
\end{itemize}
Each pass uses $O(d p + \log n)$ work-space and reuses it across keys; the number
of passes affects time but not space. Finishing all $H$ heads of slot $j$ and
concatenating gives $\mathrm{out}_j \in \mathcal{F}_p^d$ in $O(dp + \log n)$ reused
scratch.

\emph{Residual, MLP, write-back.} The residual add and the width-$d$ MLP are a
fixed circuit on the $O(dp)$-bit vector, computable in $O(dp)$ space; the result
is written to slot $j$ of a \emph{new} buffer.

\emph{Synchronous update and aggregate.} Because a sublayer may include
slot-to-slot attention, each output slot must be computed from the \emph{old} slot
values. We therefore keep two $s\,d\,p$-bit buffers, $z^{\mathrm{old}}$ and
$z^{\mathrm{new}}$: every output slot is written into $z^{\mathrm{new}}$ while
all slot reads draw from $z^{\mathrm{old}}$, and the buffers are swapped after
the sublayer (matching the $z^{\mathrm{cur}}/z^{\mathrm{new}}$ regions of
\Cref{app:upper}). This costs $2sdp = O(sdp)$ persistent space; the per-slot
scratch ($O(dp+\log n)$) is reused across the $s$ slots, so the total work-space
is $O(s\,d\,p + \log n)$.

\emph{Precision.} Every quantity is computed to precision $p$, matching the
block's definition, so the simulation outputs exactly $B_\theta(z\,;\,x)$; there
is no accumulated approximation error to control.
\end{proof}

\section{Proof of \texorpdfstring{\Cref{thm:upper}}{Theorem 4.2}}
\label{app:upper}

\noindent\textbf{\Cref{thm:upper} (Compressed loops are space-bounded).}\;\emph{For any $\CLL[s,d,p,T]$ in
the log-precision logspace-uniform regime,
$\CLL[s,d,p,T] \subseteq \DSPACE(O(s\,d\,p + \log n + \log T))$. In particular,
if $s\,d\,p = \polylog(n)$ and $\log T = \polylog(n)$, then
$\CLL \subseteq \DSPACE(\polylog n)$.}

\begin{proof}
The simulator keeps two $s\,d\,p$-bit regions $z^{\mathrm{cur}},
z^{\mathrm{new}}$ and a counter $t$ of $\lceil\log T\rceil$ bits. It writes the
seed $z_0 = \mathrm{init}_\theta(x)$ into $z^{\mathrm{cur}}$ (the initialiser,
computable in $O(sdp + \log n)$ space). For $t = 0,\dots,T-1$ it computes $z^{\mathrm{new}}
= B_\theta(z^{\mathrm{cur}}\,;\,x)$ by \Cref{lem:onestep} in $O(sdp+\log n)$
space, copies $z^{\mathrm{new}}$ over $z^{\mathrm{cur}}$, and increments $t$.
After $T$ steps it applies the decoder to $z^{\mathrm{cur}}$ ($O(sdp)$ space) and
thresholds. The input tape is re-scanned each iteration and is free. Work-space
is reused across iterations, so the total is
\[
O(s\,d\,p) + O(s\,d\,p + \log n) + O(\log T) = O(s\,d\,p + \log n + \log T).
\]
Thus the decided language lies in $\DSPACE(O(sdp + \log n + \log T))$. For
$s\,d\,p = \polylog(n)$ and $\log T = \polylog(n)$, in particular for
$T = \poly(n)$, this is $\DSPACE(\polylog n)$. With general $p$-bit arithmetic the
bound gains an additive $\poly(p)$ term, still polylogarithmic under
$p = O(\log n)$ (\Cref{app:model}).
\end{proof}

\section{Proof of \texorpdfstring{\Cref{prop:escape}}{Proposition 4.3}}
\label{app:escape}

\noindent\textbf{\Cref{prop:escape} (Sequence-state loops are memory-rich).}\;\emph{For $\SSL[d,p,T]$ the
persistent memory is $M_{\SSL} = n\,d\,p$, and
$\SSL[d,p,T] \subseteq \DSPACE(O(n\,d\,p + \log n + \log T))$, linear in the
sequence length for fixed $d,p$. Moreover, under
$\Pclass \not\subseteq \DSPACE(\polylog n)$, the class of sequence-state loop
families with polynomially many positions is not contained in
$\DSPACE(\polylog n)$.}

\begin{proof}
For $\SSL[d,p,T]$ the persistent state is the full hidden sequence $H_t \in
\mathcal{F}_p^{d\times n}$, of $n\,d\,p$ bits. Treating the $n$ positions as the slots of
\Cref{thm:upper} gives $\SSL[d,p,T]\subseteq\DSPACE(O(ndp+\log n+\log T))$,
$\Theta(n)$ for fixed $d,p$. This $O(ndp)$ bound is memory-rich rather than
tight: we do not claim every $\SSL$ needs linear space, but establish a
class-level lower bound by importing two constructions.
\citet{giannou2023programmable} realise an $\SSL$ that executes an arbitrary
instruction stream encoded in the input, and \citet{saunshi2025reasoning} show
an $\SSL$ emulates a length-$m$ CoT by appending $m$ dummy positions, so that
$\SSL$ with $m=\poly(n)$ positions and $T = \poly(n)$ iterations decides every
language reached by polynomial-length CoT. By \citet{li2024chain},
polynomial-length CoT simulates polynomial-size Boolean circuits, and by
\citet{merrill2024the}, under their transformer-decoder assumptions
polynomial-step CoT recognises exactly $\Pclass$; hence the languages reachable
by polynomial-length CoT include all of $\Pclass$, in particular the
$\Pclass$-complete CVP. If
$\SSL \subseteq \DSPACE(\polylog n)$ then $\mathrm{CVP}\in\DSPACE(\polylog n)$,
and since CVP is $\Pclass$-complete under logspace reductions this would give
$\Pclass\subseteq\DSPACE(\polylog n)$, contradicting the hypothesis. Hence
$\SSL\not\subseteq\DSPACE(\polylog n)$.
\end{proof}

\section{Proof of \texorpdfstring{\Cref{thm:crossover}}{Theorem 4.4}}
\label{app:crossover}

We first bound a compressed loop's space by its memory budget, which clarifies
why iteration count alone does not suffice.

\begin{lemma}[Memory budget upper-bounds reachable space]
\label{lem:budget}
Write $\mathrm{LOOP}[M]$ for a compressed loop with $M = s\,d\,p$ bits of
persistent state, log-precision, and $T$ iterations. Then
$\mathrm{LOOP}[M] \subseteq \DSPACE(O(M + \log n + \log T))$; for $T = \poly(n)$
the counter costs $\log T = O(\log n)$ and the bound becomes
$\DSPACE(O(M + \log n))$.
\end{lemma}

\begin{proof}
This is \Cref{thm:upper} with $M = sdp$, the iteration counter staying within
the same budget.
\end{proof}

\begin{remark}[Converse intuition]
Conversely, a deterministic computation using space $w(n)$ and halting within
$\tau$ steps can be represented by a $\mathrm{LOOP}[M]$ with $M = O(w(n))$ and
$T = \tau$, provided the block realises the machine's transition map. We use this
only as intuition for why $M$, rather than $T$ alone, is the governing resource;
the separation below uses only the forward inclusion.
\end{remark}

\begin{sloppypar}
\noindent\textbf{\Cref{thm:crossover} (Compressed-loop/chain-of-thought separation).} \emph{Assume
$\Pclass \not\subseteq \DSPACE(\polylog n)$, and let $\mathcal{L}$ be
$\Pclass$-complete under logspace reductions. No logspace-uniform compressed-loop family with
$s\,d\,p = \polylog(n)$ and $\log T = \polylog(n)$ decides $\mathcal{L}$; by
contrast, under standard expressivity results $\mathcal{L}$ is decidable by a
chain-of-thought family with polynomially many generated tokens.}
\end{sloppypar}

\begin{proof}
For $s\,d\,p=\polylog n$ and $\log T=\polylog n$,
$\mathrm{LOOP}[M]\subseteq\DSPACE(\polylog n)$ by \Cref{lem:budget}, the
iteration counter staying within the same budget. Let $\mathcal{L}$ be
$\Pclass$-complete under logspace reductions. If such a compressed loop decided
$\mathcal{L}$, then $\mathcal{L}\in\DSPACE(\polylog n)$. Every language in $\Pclass$ reduces to
$\mathcal{L}$ by a logspace-computable reduction $f$ of polynomial output length,
and the composition stays in polylogarithmic space, since each queried bit of
$f(x)$ is recomputed in $O(\log n)$ space and $\polylog(|f(x)|) = \polylog(n)$ as
$|f(x)|$ is polynomial in $n$. Hence $\Pclass\subseteq\DSPACE(\polylog n)$,
contradicting the assumption. Under standard chain-of-thought expressivity
results, polynomial-length CoT decides $\mathcal{L}$: \citet{li2024chain} give
the simulation of size-$T$ Boolean circuits by $T$-step CoT, and
\citet{merrill2024the} characterise the polynomial-step regime as exactly
$\Pclass$ under their model assumptions. Since $\mathcal{L}\in\Pclass$, this
gives the separation.
\end{proof}

\begin{remark}[Endpoints and the governing quantity]
More generally, by the forward inclusion of \Cref{lem:budget} a language with
deterministic space lower bound $w(n)$ cannot be decided by $\mathrm{LOOP}[M]$
whenever $M + \log n + \log T = o(w(n))$; for $\poly(n)$ iterations and
$w(n) \gg \log n$ this requires $M = \Omega(w(n))$, while CoT supplies
$\Theta(\ell)$ growing scratch and reaches such tasks at $\ell=\poly(n)$. The relative
power is therefore governed by $M$ versus $w(n)$. At the memory-rich endpoint
$M = \Theta(ndp)$ ($\Theta(n)$ for fixed $d,p$), $\mathrm{LOOP}[M]$ is the
sequence-state loop, the latent-thought regime discussed in \Cref{sec:related}; our separation concerns the opposite,
memory-constrained endpoint.
\end{remark}

\section{Experimental details and seed-level results}
\label{app:seeds}

This appendix reports the seed-level results underlying
\Cref{fig:heatmap,fig:gla}. The tables are included only for reproducibility:
every mean and standard deviation in \Cref{sec:empirics} is computed over the
listed seeds, with no failed, low, or outlier run excluded. Labels such as
``failed seed'', ``high seed'', and ``low seed'' are post-hoc descriptive terms
for interpreting optimisation regimes, not filtering criteria.

\paragraph{Experimental setup.}
Both sweeps train with Adam at learning rate $3\times 10^{-4}$, batch size $64$,
for $12{,}000$ steps, and evaluate exact accuracy on $5120$ held-out examples
($20$ batches of $256$). Pointer chasing uses functional graphs on $n=16$ cells
with $4$-hop chains, hidden width $d=64$, $4$ attention heads, and $T=16$ loop
iterations; the compressed loop sweeps $s \in \{1,2,4,8,16\}$ slots against
$k \in \{1,2,4,8\}$ chains, and an instance is correct only if all $k$ endpoints
match. Associative recall uses inline multi-query associative recall with
vocabulary $K=64$, four queries per example, width $d=128$ and four blocks, and a
gated linear-attention block with a short causal depthwise convolution, sweeping
state dimension $m \in \{4,8,16,32,64,128\}$ against $N \in \{8,16,24\}$
associations and scoring per-query-token accuracy. Seeds $0$--$9$ (pointer
chasing) and $0$--$4$ (associative recall) are independent random initialisations
and data orders.

\subsection{Pointer chasing}
\label{app:seeds-poc}

\Cref{tab:pointer-cll-seeds} gives every compressed-loop cell, and
\Cref{tab:pointer-ssl-seeds} the sequence-state control, each over the ten seeds.

\begin{table}[H]
\centering\footnotesize\setlength{\tabcolsep}{3.2pt}
\begin{tabular}{ccrrrrrrrrrrcc}
\toprule
$k$ & $s$ & \multicolumn{10}{c}{seed} & mean & std \\
\cmidrule(lr){3-12}
 & & 0 & 1 & 2 & 3 & 4 & 5 & 6 & 7 & 8 & 9 & & \\
\midrule
$1$ & $1$ & 0.86 & 0.37 & 0.43 & 0.94 & 0.92 & 0.34 & 0.93 & 0.93 & 0.40 & 0.90 & 0.70 & 0.26 \\
$1$ & $2$ & 0.99 & 0.92 & 0.95 & 0.96 & 0.95 & 0.90 & 0.92 & 0.77 & 0.96 & 0.94 & 0.93 & 0.06 \\
$1$ & $4$ & 0.94 & 0.94 & 0.98 & 0.95 & 0.98 & 0.99 & 0.91 & 0.86 & 0.97 & 0.93 & 0.94 & 0.04 \\
$1$ & $8$ & 0.97 & 0.99 & 0.92 & 0.94 & 0.90 & 0.90 & 0.94 & 0.94 & 0.87 & 0.43 & 0.88 & 0.15 \\
$1$ & $16$ & 0.92 & 0.89 & 0.36 & 0.87 & 0.99 & 0.86 & 0.95 & 0.85 & 0.86 & 0.86 & 0.84 & 0.17 \\
\addlinespace[1pt]
$2$ & $1$\;{\scriptsize$(s<k)$} & 0.11 & 0.12 & 0.16 & 0.13 & 0.11 & 0.12 & 0.13 & 0.12 & 0.15 & 0.14 & 0.13 & 0.02 \\
$2$ & $2$ & 1.00 & 1.00 & 1.00 & 1.00 & 1.00 & 1.00 & 1.00 & 0.99 & 1.00 & 1.00 & 1.00 & 0.00 \\
$2$ & $4$ & 1.00 & 1.00 & 1.00 & 1.00 & 1.00 & 1.00 & 0.99 & 1.00 & 1.00 & 0.98 & 1.00 & 0.00 \\
$2$ & $8$ & 1.00 & 1.00 & 0.99 & 1.00 & 1.00 & 1.00 & 1.00 & 1.00 & 1.00 & 1.00 & 1.00 & 0.00 \\
$2$ & $16$ & 1.00 & 0.92 & 1.00 & 1.00 & 1.00 & 0.84 & 0.78 & 0.92 & 1.00 & 0.92 & 0.94 & 0.07 \\
\addlinespace[1pt]
$4$ & $1$\;{\scriptsize$(s<k)$} & 0.03 & 0.04 & 0.04 & 0.03 & 0.03 & 0.03 & 0.03 & 0.03 & 0.03 & 0.03 & 0.03 & 0.00 \\
$4$ & $2$\;{\scriptsize$(s<k)$} & 0.06 & 0.05 & 0.06 & 0.05 & 0.05 & 0.03 & 0.05 & 0.05 & 0.05 & 0.05 & 0.05 & 0.01 \\
$4$ & $4$ & 1.00 & 1.00 & 1.00 & 1.00 & 1.00 & 1.00 & 1.00 & 1.00 & 0.99 & 0.99 & 1.00 & 0.00 \\
$4$ & $8$ & 1.00 & 1.00 & 1.00 & 1.00 & 1.00 & 1.00 & 1.00 & 1.00 & 1.00 & 0.99 & 1.00 & 0.00 \\
$4$ & $16$ & 0.45 & 0.53 & 0.89 & 0.50 & 0.50 & 0.51 & 0.52 & 0.52 & 0.52 & 0.44 & 0.54 & 0.12 \\
\addlinespace[1pt]
$8$ & $1$\;{\scriptsize$(s<k)$} & 0.01 & 0.01 & 0.01 & 0.01 & 0.01 & 0.01 & 0.01 & 0.01 & 0.01 & 0.01 & 0.01 & 0.00 \\
$8$ & $2$\;{\scriptsize$(s<k)$} & 0.01 & 0.01 & 0.01 & 0.01 & 0.01 & 0.01 & 0.01 & 0.01 & 0.01 & 0.01 & 0.01 & 0.00 \\
$8$ & $4$\;{\scriptsize$(s<k)$} & 0.01 & 0.02 & 0.02 & 0.02 & 0.01 & 0.01 & 0.02 & 0.02 & 0.01 & 0.02 & 0.01 & 0.00 \\
$8$ & $8$ & 0.04 & 0.05 & 0.05 & 0.04 & 0.04 & 0.05 & 0.04 & 0.04 & 0.04 & 0.05 & 0.04 & 0.00 \\
$8$ & $16$ & 0.73 & 0.81 & 0.36 & 0.99 & 0.74 & 0.14 & 0.31 & 0.11 & 0.70 & 0.54 & 0.54 & 0.29 \\
\bottomrule
\end{tabular}
\caption{Seed-level exact-match accuracy for the compressed-loop pointer-chasing sweep (Figure~\ref{fig:heatmap}). Columns $0$--$9$ are the ten seeds; the mean and standard deviation are over all ten, with no seed excluded.}
\label{tab:pointer-cll-seeds}
\end{table}

\begin{table}[H]
\centering\footnotesize\setlength{\tabcolsep}{3.2pt}
\begin{tabular}{crrrrrrrrrrcc}
\toprule
$k$ & \multicolumn{10}{c}{seed} & mean & std \\
\cmidrule(lr){2-11}
 & 0 & 1 & 2 & 3 & 4 & 5 & 6 & 7 & 8 & 9 & & \\
\midrule
$1$ & 1.00 & 1.00 & 1.00 & 0.99 & 1.00 & 0.38 & 1.00 & 1.00 & 1.00 & 1.00 & 0.94 & 0.19 \\
$2$ & 1.00 & 1.00 & 1.00 & 1.00 & 1.00 & 1.00 & 1.00 & 1.00 & 1.00 & 1.00 & 1.00 & 0.00 \\
$4$ & 1.00 & 0.99 & 1.00 & 1.00 & 1.00 & 1.00 & 1.00 & 1.00 & 1.00 & 1.00 & 1.00 & 0.00 \\
$8$ & 1.00 & 1.00 & 1.00 & 1.00 & 1.00 & 1.00 & 1.00 & 1.00 & 1.00 & 1.00 & 1.00 & 0.00 \\
\bottomrule
\end{tabular}
\caption{Seed-level exact-match accuracy for the sequence-state loop control (Figure~\ref{fig:heatmap}). The single $k=1$ seed that failed to train is included in the mean and standard deviation.}
\label{tab:pointer-ssl-seeds}
\end{table}

\subsection{Associative recall}
\label{app:seeds-gla}

\Cref{tab:gla-seeds} gives the gated-linear-attention cells, \Cref{tab:mcrit-seeds}
the per-seed critical state dimension $m_{\mathrm{crit}}$, and
\Cref{tab:gla-controls} the softmax and vanilla-linear controls, each over the
five seeds.

\begin{table}[H]
\centering\small
\begin{tabular}{ccrrrrrcc}
\toprule
$N$ & $m$ & \multicolumn{5}{c}{seed} & mean & std \\
\cmidrule(lr){3-7}
 & & 0 & 1 & 2 & 3 & 4 & & \\
\midrule
$8$ & $4$ & 0.81 & 0.92 & 0.69 & 0.86 & 0.93 & 0.84 & 0.09 \\
$8$ & $8$ & 0.98 & 0.98 & 0.77 & 0.95 & 0.95 & 0.93 & 0.08 \\
$8$ & $16$ & 0.98 & 0.96 & 0.99 & 0.99 & 0.99 & 0.98 & 0.01 \\
$8$ & $32$ & 0.99 & 1.00 & 0.99 & 1.00 & 1.00 & 1.00 & 0.00 \\
$8$ & $64$ & 1.00 & 1.00 & 1.00 & 1.00 & 1.00 & 1.00 & 0.00 \\
$8$ & $128$ & 1.00 & 1.00 & 1.00 & 1.00 & 1.00 & 1.00 & 0.00 \\
\addlinespace[1pt]
$16$ & $4$ & 0.47 & 0.20 & 0.16 & 0.12 & 0.17 & 0.23 & 0.13 \\
$16$ & $8$ & 0.43 & 0.62 & 0.51 & 0.17 & 0.36 & 0.42 & 0.15 \\
$16$ & $16$ & 0.69 & 0.92 & 0.35 & 0.47 & 0.92 & 0.67 & 0.23 \\
$16$ & $32$ & 0.96 & 0.78 & 0.93 & 0.84 & 0.94 & 0.89 & 0.07 \\
$16$ & $64$ & 0.99 & 0.98 & 0.98 & 0.98 & 0.98 & 0.98 & 0.00 \\
$16$ & $128$ & 1.00 & 0.99 & 0.99 & 0.99 & 0.99 & 0.99 & 0.00 \\
\addlinespace[1pt]
$24$ & $4$ & 0.13 & 0.12 & 0.09 & 0.09 & 0.09 & 0.10 & 0.02 \\
$24$ & $8$ & 0.42 & 0.23 & 0.16 & 0.11 & 0.10 & 0.20 & 0.12 \\
$24$ & $16$ & 0.43 & 0.59 & 0.16 & 0.58 & 0.16 & 0.38 & 0.19 \\
$24$ & $32$ & 0.84 & 0.66 & 0.95 & 0.72 & 0.83 & 0.80 & 0.10 \\
$24$ & $64$ & 0.88 & 0.91 & 0.96 & 0.94 & 0.89 & 0.91 & 0.03 \\
$24$ & $128$ & 0.98 & 0.94 & 0.98 & 0.93 & 0.99 & 0.96 & 0.03 \\
\bottomrule
\end{tabular}
\caption{Seed-level query-token accuracy for gated linear attention on associative recall (Figure~\ref{fig:gla}). Columns $0$--$4$ are the five seeds; the mean and standard deviation are over all five.}
\label{tab:gla-seeds}
\end{table}

\begin{table}[H]
\centering\small
\begin{tabular}{crrrrrcc}
\toprule
$N$ & \multicolumn{5}{c}{seed} & median & range \\
\cmidrule(lr){2-6}
 & 0 & 1 & 2 & 3 & 4 & & \\
\midrule
$8$ & 8 & 4 & 16 & 8 & 4 & 8 & [4,16] \\
$16$ & 32 & 16 & 32 & 64 & 16 & 32 & [16,64] \\
$24$ & 128 & 64 & 32 & 64 & 128 & 64 & [32,128] \\
\bottomrule
\end{tabular}
\caption{Per-seed critical state dimension $m_{\mathrm{crit}}$, the smallest $m$ whose query-token accuracy is at least $0.9$ (Figure~\ref{fig:gla}).}
\label{tab:mcrit-seeds}
\end{table}

\begin{table}[H]
\centering\small
\begin{tabular}{llc}
\toprule
model & setting & accuracy \\
\midrule
full softmax attention & $N=8$ & $1.00 \pm 0.00$ \\
full softmax attention & $N=16$ & $1.00 \pm 0.00$ \\
full softmax attention & $N=24$ & $1.00 \pm 0.00$ \\
\addlinespace[1pt]
vanilla linear attention & $N=8$, best $m=128$ & $0.98 \pm 0.02$ \\
vanilla linear attention & $N=16$, best $m=8$ & $0.11 \pm 0.00$ \\
vanilla linear attention & $N=24$, best $m=16$ & $0.09 \pm 0.00$ \\
\bottomrule
\end{tabular}
\caption{Control summaries for associative recall, mean $\pm$ std over the five seeds. Softmax attention is the memory-rich control (independent of $m$); vanilla linear attention is the compressed-state negative control, reported at the best $m$ in the sweep for each $N$.}
\label{tab:gla-controls}
\end{table}

\end{document}